\documentclass[11pt]{article}
\usepackage{eacl2017}
\usepackage{times}
\usepackage{latexsym}
\usepackage[dvipdfmx]{graphicx}
\usepackage{amsmath}
\usepackage{amssymb}
\usepackage{multirow}
\usepackage{url}
\usepackage{algorithm}
\usepackage{algorithmic}
\usepackage{amsthm}

\eaclfinalcopy
\DeclareMathOperator*{\argmax}{arg\,max}

\newtheorem{Thm}{Theorem}
\newtheorem*{Thm*}{Theorem}

\title{Enumeration of Extractive Oracle Summaries}

\author{
Tsutomu Hirao \and Masaaki Nishino \and Jun Suzuki \and Masaaki Nagata\\
NTT Communication Science Laboratories, NTT Corporation\\
2-4 Hikaridai, Seika-cho, Soraku-gun, Kyoto, 619-0237, Japan\\
{\tt \{hirao.tsutomu,nishino.masaaki\}@lab.ntt.co.jp }\\
{\tt \{suzuki.jun,nagata.masaaki\}@lab.ntt.co.jp }\\
}

\date{}

\begin{document}

\maketitle
\begin{abstract}
To analyze the limitations and the future directions of the extractive
 summarization paradigm, 
this paper proposes an Integer Linear Programming (ILP) formulation to
 obtain {\it extractive oracle summaries} in terms of $\text{\sc Rouge}_n$.
We also propose an algorithm that enumerates 
 all of the oracle summaries for a set of reference summaries to exploit F-measures that
 evaluate which system summaries contain how many sentences
 that are extracted as an oracle summary.
 Our experimental results obtained from Document Understanding Conference (DUC) corpora demonstrated the following:
(1) room still exists to 
 improve the performance of extractive summarization; 
(2) the F-measures derived from the enumerated oracle summaries have
 significantly stronger
 correlations with human judgment than those derived from
 single oracle summaries.
\end{abstract}

\section{Introduction}

Recently, compressive and abstractive summarization are attracting attention (e.g., \newcite{Almeida13}, \newcite{Qian13},
\newcite{Yao:IJCAI15}, \newcite{Banerjee:IJCAI15}, \newcite{Bing15}).
However, extractive summarization remains a primary research topic
because the linguistic quality of the resultant summaries
is guaranteed, at least at the sentence level, which is a key requirement for
practical use
(e.g., \newcite{hong-nenkova:2014:EACL},
\newcite{hong-marcus-nenkova:2015:EMNLP},
\newcite{yogatama-liu-smith:2015:EMNLP}, \newcite{parveen-ramsl-strube:2015:EMNLP}).

The summarization research community is experiencing a paradigm
shift from extractive to compressive or abstractive summarization.
Currently our question is: ``Is extractive summarization still useful research?''
To answer it, 
the ultimate limitations of the extractive
summarization paradigm must be comprehended;
that is, we have to determine its upper bound and compare it with the performance of the
state-of-the-art summarization methods.
Since $\text{\sc Rouge}_n$ is the de-facto automatic evaluation method
and is employed in many text summarization studies, an oracle summary is
defined as a set of sentences that have a maximum $\text{\sc Rouge}_n$ score.
If the $\text{\sc Rouge}_n$ score of an oracle summary
 outperforms that of a system that employs
another summarization approach, the extractive summarization paradigm
is worthwhile to leverage research resources.

As another benefit,
identifying an oracle summary for a
set of reference summaries allows us to utilize yet another evaluation measure.
Since both oracle and extractive summaries are sets of sentences, 
it is easy to check whether a system summary contains sentences in the 
oracle summary.
As a result, F-measures, which are
available to evaluate a system summary,
are useful for evaluating classification-based extractive
summarization \cite{Mani:1998,osborne:2002,hirao02}.
Since $\text{\sc Rouge}_n$ evaluation does not identify
which sentence is important, an F-measure conveys useful information in
terms of ``important sentence extraction.''  
Thus, combining  $\text{\sc Rouge}_n$ and an F-measure allows us
to scrutinize the failure analysis of systems.

Note that more than one oracle summary might exist for a set of reference
summaries because $\text{\sc
Rouge}_n$ scores are based on the unweighted counting of n-grams.
As a result, an F-measure might not be identical among multiple oracle summaries.
Thus, we need to enumerate the oracle summaries for a set of reference summaries and 
compute the F-measures based on them.

In this paper,  we first derive an Integer Linear
 Programming (ILP) problem to extract an oracle summary from a set of reference summaries
 and a source document(s).
To the best of our knowledge, this is the first ILP formulation that extracts
 oracle summaries.
Second, since it is difficult to enumerate oracle summaries for a set of reference
 summaries using ILP solvers,
 we propose an algorithm that efficiently enumerates 
 all oracle summaries by exploiting the branch and bound technique.
Our experimental results on the Document Understanding Conference (DUC)
corpora showed the following:
\begin{enumerate}
 \item Room still exists for the further improvement of extractive
       summarization, {\it i.e.}, where the
      $\text{\sc Rouge}_n$ scores of the oracle summaries are
       significantly higher than those of the state-of-the-art summarization
       systems.
 \item
      The F-measures derived from multiple oracle summaries obtain
      significantly stronger correlations with human judgment than those
      derived from single oracle summaries.
\end{enumerate}

\section{Definition of Extractive Oracle Summaries}

We first briefly describe $\text{\sc Rouge}_n$.
Given set of reference summaries $\boldsymbol{R}$ and system summary $S$, $\text{\sc Rouge}_n$
is defined as follows:
\begin{equation}
\label{rouge}
\begin{split}
\text{\sc Rouge}_n(\boldsymbol{R},S)=&\\
&\kern-7em \frac{ \displaystyle\sum_{k=1}^{|\boldsymbol{R}|}\sum_{j=1}^{|U({\cal R}_k)|}\min\{N(g_j^n,{\cal R}_k),N(g_j^n,{\cal
 S})\}}{\displaystyle\sum_{k=1}^{|\boldsymbol{R}|} \sum_{j=1}^{|U({\cal R}_k)|} N(g_j^n,{\cal R}_k)}.
\end{split}
\end{equation}

\noindent ${\cal R}_k$ denotes the multiple set of n-grams that occur in
$k$-th reference summary $R_k$, and  $\mathcal{S}$ denotes the multiple set of n-grams
that appear in system-generated summary $S$ (a set of sentences).
$N(g_j^n,{\cal R}_k)$ and $N(g_j^n,{\cal S})$ return the number of occurrences of 
n-gram $g_j^n$ in the $k$-th reference and system summaries, respectively.
Function $U(\cdot)$ transforms a multiple set into a normal set.
$\text{\sc Rouge}_n$ takes values in the range of $[0,1]$, and when the n-gram occurrences of the system summary agree with those of the reference summary, the value is 1. 

In this paper, we focus on extractive summarization, employ $\text{\sc Rouge}_n$ as an evaluation measure, and define the 
oracle summaries as follows:
\begin{equation}
\label{def:oracle}
\begin{split}
O=&\argmax_{S \subseteq D} \text{\sc Rouge}_n(\boldsymbol{R},S)\\
s.t.& ~~~\ell(S) \le L_{\rm max}. 
\end{split}
\end{equation}

$D$ is the set of all the sentences contained in the input document(s), and
$L_{\rm max}$
is the length limitation of the oracle summary.
 $\ell(S)$ indicates
the number of words in the system summary.
Eq. (\ref{def:oracle}) is an NP-hard combinatorial optimization problem, and no polynomial time algorithms exist that can attain an optimal solution.

\section{Related Work}

\newcite{Lin-Hovy:2003:DUC} utilized a naive exhaustive search method to
obtain oracle summaries in terms of $\text{\sc Rouge}_n$
 and exploited them to understand the
limitations of extractive summarization systems. 
\newcite{ceylan-EtAl:2010:NAACLHLT}
proposed another naive exhaustive search method to derive a probability density
function from the {\sc Rouge}$_n$ scores of oracle summaries for the domains to which source documents belong.
The computational complexity of naive exhaustive methods is exponential to the size of the sentence set. Thus, it may be possible to apply them to single document
summarization tasks involving a dozen sentences, but it is infeasible to apply them to multiple document summarization tasks that involve several hundred sentences.

To describe the difference between the $\text{\sc Rouge}_n$ scores of oracle
and system summaries in multiple document summarization tasks,
\newcite{Riedhammer} proposed an approximate algorithm
 with a genetic algorithm (GA) to find oracle
 summaries. \newcite{Moen14}
 utilized a greedy
algorithm for the same purpose. 
Although GA or greedy algorithms are widely used to solve NP-hard
combinatorial optimization problems, the solutions are not always optimal. Thus, the summary does not always
have a maximum $\text{\sc Rouge}_n$ score for the set of reference summaries.
Both works called the summary found by their methods the oracle, but it differs from the definition in our paper.

Since summarization systems cannot reproduce human-made reference summaries in
most cases, oracle summaries, which can be reproduced by summarization
systems, have been used as training data to tune the parameters of
summarization systems.
For example, \newcite{Kulesza} and \newcite{sipos-shivaswamy-joachims:2012:EACL2012}
trained their summarizers with oracle summaries found by a
greedy algorithm.
\newcite{peyrard-ecklekohler:2016} proposed a method to find a summary that
approximates a {\sc Rouge} score based on the {\sc Rouge}
scores of individual sentences and exploited the framework to train their summarizer.
As mentioned above, such summaries do not always
agree with the oracle summaries defined in our paper. Thus, the quality of the training data is suspect.
Moreover, since these studies fail to consider that a set of reference summaries has
multiple oracle summaries, the score of the loss function defined between
their oracle and system summaries is not appropriate in most cases.

As mentioned above, no known efficient algorithm can extract ``exact'' oracle summaries, as defined in Eq. (2), {\it i.e.}, because only a naive exhaustive search is available.
Thus, such approximate algorithms as a greedy algorithm are mainly employed to obtain them.

\section{Oracle Summary Extraction as an Integer Linear Programming (ILP) Problem }

To extract an oracle summary from document(s) and a given set of reference summaries, we start by deriving an Integer Linear Programming (ILP) problem.
Since the denominator of Eq. (\ref{rouge}) is constant for a given
set of reference summaries, we can find an oracle summary by maximizing
the numerator of Eq.
(\ref{rouge}). Thus, the ILP formulation is defined 
 as follows:
\begin{eqnarray}
   {\displaystyle\mathop\text{maximize}_{\boldsymbol{z}}}
  & \displaystyle\sum_{k=1}^{|\boldsymbol{R}|}\sum_{j=1}^{|U({\cal R}_k)|} z_{kj}\\
 s.t. & \displaystyle\sum_{i=1}^{|D|} \ell(s_i) x_i \le L_{\rm max}\\
      & \forall j : \displaystyle\sum_{i=1}^{|D|} N(g_{j}^n,s_i) x_i \ge z_{kj}\\
      & \forall j : N(g_j^n,{\cal R}_k) \ge z_{kj}\\
      & \displaystyle\forall i : x_i \in \{0,1\}\\
      & \displaystyle\forall j : z_{kj} \in \mathbb{Z}_{+}.
\end{eqnarray}

Here, $z_{kj}$ is the count of the $j$-th n-gram of the $k$-th reference
summary in the oracle summary, {\it
i.e.}, $z_{kj}=\min\{N(g_j^n,{\cal R}_k), N(g_j^n,{\cal S})\}$. $\ell(\cdot)$
returns the number of words in the sentence, $x_i$ is a
binary indicator, and $x_i=1$ denotes that the $i$-th sentence $s_i$ is
included in the oracle summary.
$N(g_j^n,s_i)$ returns the number of occurrences of 
n-gram $g_j^n$ in the $i$-th sentence. Constraints (5) and (6)
 ensure that $z_{kj}=\min\{N(g_j^n,{\cal R}_k), N(g_j^n,{\cal S})\}$.

\section{Branch and Bound Technique for Enumerating Oracle Summaries}

Since enumerating oracle summaries with an ILP solver is difficult,
we extend the exhaustive search approach by introducing a search and
prune technique to enumerate the oracle summaries.
The search pruning decision is made by comparing the current upper bound of
the {\sc Rouge}$_n$ score with the maximum {\sc Rouge}$_n$ score in the
search history.

\subsection{$\text{\sc Rouge}_n$ Score for Two Distinct Sets of Sentences}

\begin{figure}[tb]
 \begin{center}
  \includegraphics[keepaspectratio, scale=0.3]{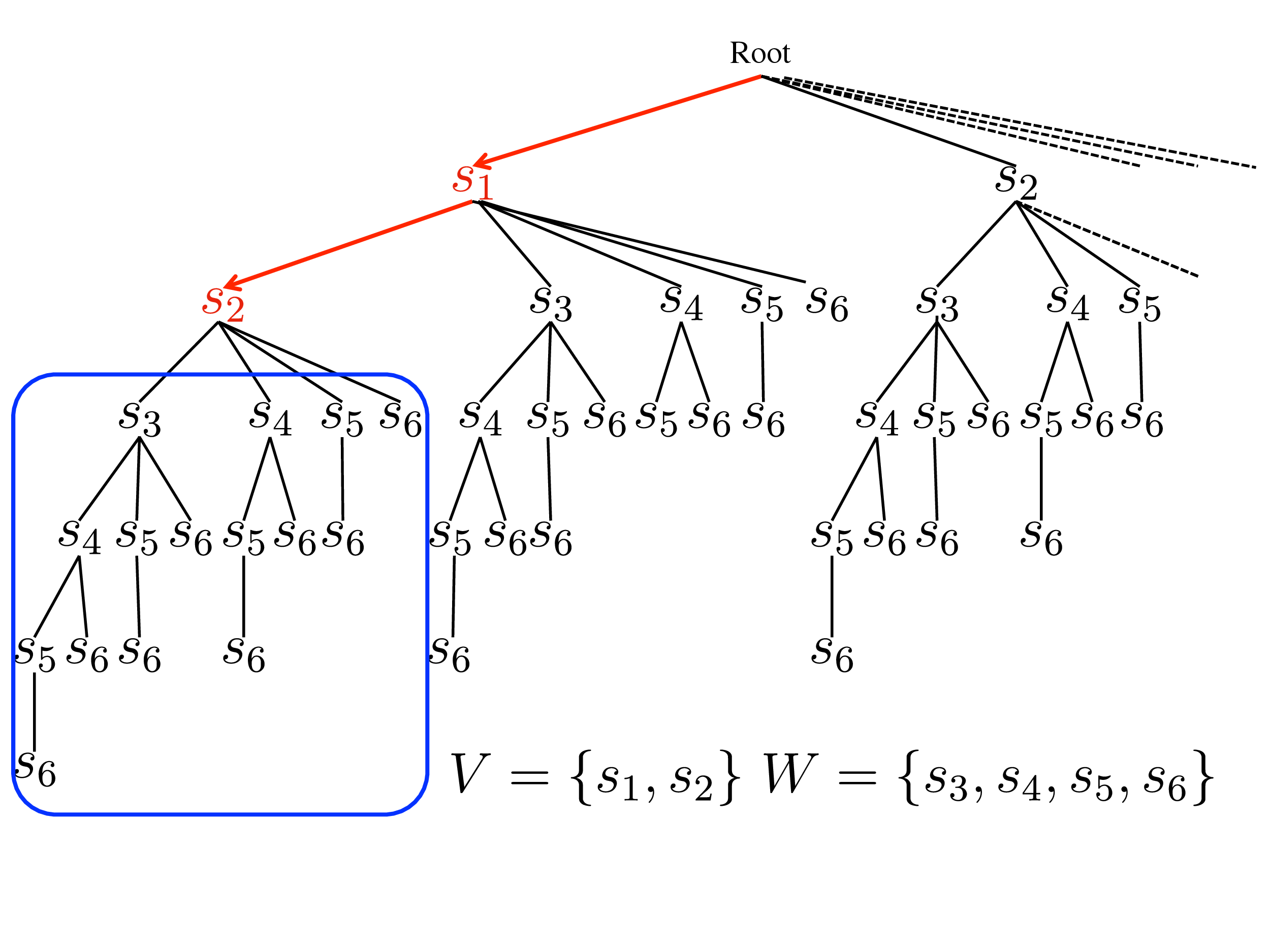}
  \caption{Example of a search tree}
  \label{search_tree}
 \end{center}
\end{figure}

The enumeration of oracle summaries can be regarded as a depth-first
search on a tree whose nodes represent sentences.
Fig. \ref{search_tree} shows an example of a search tree created in a naive
exhaustive search. The nodes represent
sentences and the path from the root node to an arbitrary
node represents a summary. For example, the red path in Fig. \ref{search_tree} from the root node to node $s_2$ represents a
summary consisting of sentences $s_1,s_2$.
By utilizing the tree, we can enumerate oracle summaries
by exploiting depth-first searches while excluding the summaries that violate
length constraints. 
However, this naive exhaustive search approach is impractical for large data sets because the number of nodes inside the
tree is $2^{|D|}$.

If we prune the unwarranted subtrees in each step of the
depth-first search, we can make the search more efficient.
The decision to search or prune is made by comparing the current upper bound of the {\sc Rouge}$_n$ score with the maximum {\sc Rouge}$_n$ score in the search history. For instance, in Fig. \ref{search_tree}, we reach node $s_2$ by following this path: ``Root $\rightarrow$ $s_1$, $\rightarrow$ $s_2$''. 
If we estimate the maximum $\text{\sc Rouge}_n$
score (upper bound) obtained by searching for the descendant of 
$s_2$ (the subtree in the blue rectangle),
we can decide whether the depth-first search should be continued. When the upper bound of the $\text{\sc Rouge}_n$ score exceeds the current maximum $\text{\sc Rouge}_n$ in the search history, we have to continue.
When the upper bound is smaller than the current maximum
$\text{\sc Rouge}_n$ score, no summary is optimal that contains $s_1$, $s_2$, so we can skip subsequent search activity on the subtree and proceed to check the next branch:
``Root $\rightarrow$  $s_1$
$\rightarrow$ $s_3$''.

To estimate the upper bound of the $\text{\sc Rouge}_n$ score,
 we re-define it for two distinct sets of sentences, $V$ and $W$, {\it i.e.,} $V \cap W=\phi$, as follows:
\begin{equation}
\label{our_rouge}
\begin{split}
\text{\sc Rouge}_n(\boldsymbol{R},V{\cup}W)&=\text{\sc Rouge}_n(\boldsymbol{R},V)\\
 &\quad +\text{\sc Rouge}'_n(\boldsymbol{R},V,W).
\end{split}
\end{equation}

\noindent Here $\text{\sc Rouge}'_n$ is defined as follows:
\begin{equation}
\begin{split}
\text{\sc Rouge}'_n(\boldsymbol{R},V,W)=&\\
& \kern-9em \frac{\displaystyle\sum_{k=1}^{|\boldsymbol{R}|}\sum_{t_n \in U({\cal R}_k)}
\min\{N(t_n,{{\cal R}_k \setminus {\cal V}}),N(t_n,{\cal W})\}}{\displaystyle\sum_{k=1}^{|\boldsymbol{R}|}\sum_{t_n
 \in U({\cal R}_k))} N(t_n,{\cal R}_k)}.
\end{split}
\end{equation}

\noindent ${\cal V,W}$ are the multiple sets of n-grams found in the sets
of sentences $V$ and $W$, respectively.
\begin{Thm}
Eq. (\ref{our_rouge}) is correct.
\end{Thm}

\begin{proof}
See Appendix A. 
\end{proof}

\begin{algorithm}[tb]
 \caption{Algorithm to Find Upper Bound of {\sc Rouge}$_n$}
 \label{findupper}
\begin{footnotesize}
 \begin{algorithmic}[1]
  \FUNC {$\widehat{\text{\sc Rouge}_n}(\boldsymbol{R},V$)}
  \STATE $W \leftarrow {\rm descendant}({\rm last}(V))$,~~$W' \leftarrow \phi$
  \STATE $U \leftarrow \mbox{\sc Rouge}(\boldsymbol{R},V)$
  \FOR {{\bf each } $w \in W$}
  \STATE append($W', \frac{\text{\sc Rouge}'_n(\boldsymbol{R},V,\{w\})}{\ell(w)}$)
  \ENDFOR 
  \STATE sort($W',\text{'descend'}$)
  \FOR {{\bf each } $w \in W'$}
  \IF {$L_{\rm max} - \ell(\{w\}) \ge 0$}
  \STATE $U \leftarrow U+\text{\sc Rouge}'_n(\boldsymbol{R},V,\{w\})$
  \STATE $L_{\rm max} \leftarrow L_{\rm max}-\ell(\{w\})$
  \ELSE
  \STATE  $\displaystyle U \leftarrow U+\frac{\text{\sc
  Rouge}'_n(\boldsymbol{R},V,\{w\})}{\ell(\{w\})} \times L_{\rm max}$
  \STATE break the loop
  \ENDIF
  \ENDFOR
  \STATE {\bf return} $U$
  \ENDFUNC
 \end{algorithmic}
\end{footnotesize}
\end{algorithm}

\subsection{Upper Bound of $\text{\sc Rouge}_n$}

Let $V$ be the set of sentences on the path from the current node to
the root node in the search tree, and let $W$ be the set of sentences 
that are the descendants of the current node. 
In Fig. \ref{search_tree}, $V{=}\{s_1,s_2\}$ and
$W{=}\{s_3,s_4,s_5,s_6\}$.
According to Theorem 1, the upper bound of the $\text{\sc Rouge}_n$ score is defined as:
\begin{eqnarray}
\label{maxrouge}
& \kern-5em \widehat{\text{\sc Rouge}_n}(\boldsymbol{R},V)=\text{\sc Rouge}_n(\boldsymbol{R},V)+\nonumber\\ 
&\kern-2em \displaystyle\mathop{\max_{\Omega \subseteq W}}\{\text{\sc
 Rouge}_n'(\boldsymbol{R},V,\Omega){:}\ell(\Omega){\le} L_{\rm max}{-}\ell(V) \}.
\end{eqnarray}

Since the second term on the right side in Eq. (\ref{maxrouge}) is
an NP-hard problem, we turn to the following relation by introducing
inequality, $\text{\sc Rouge}'_n(\boldsymbol{R},V,\Omega) \le \sum_{\omega \in \Omega}
\text{\sc Rouge}'_n(\boldsymbol{R},V,\{\omega\})$,
\begin{eqnarray}
\label{inequality}
\small
&\kern-1em\displaystyle\mathop{\max_{\Omega \subseteq W}} \left\{ \text{\sc
						  Rouge}'_n(\boldsymbol{R},V,\Omega){:}
						  \ell(\Omega) {\le}
						  L_{\rm max}{-}\ell(V)\right\} \nonumber\\
&\kern-2em {\le} {\displaystyle\max_{\boldsymbol{x}}} \left\{ \sum_{i=1}^{|W|}\text{\sc
Rouge}'_n(\boldsymbol{R}, V,\{w_i\}) x_i{:} \right. \nonumber\\
&\kern-5em \left. {\sum_{i=1}^{|W|}}\ell(\{w_i\}) x_i{\le}
  L_{\rm max}{-}\ell(V) \right\}.
\end{eqnarray}

\noindent Here, $\mathbf{x} = (x_1, \ldots, x_{|W|})$ and $x_i \in \{0,1\}$. 
The right side of Eq. (\ref{inequality}) is a knapsack problem,
{\it i.e.}, a 0-1 ILP problem. 
Although we can obtain the optimal solution for it using dynamic programming or ILP solvers, we solve its linear programming relaxation version by applying a greedy algorithm for greater computation efficiency.
The solution output by the greedy algorithm is optimal for the relaxed problem.
Since the optimal solution of the relaxed problem is always
larger than that of the original problem, the relaxed problem solution can be utilized as the upper bound.
Algorithm
\ref{findupper} shows the pseudocode that attains the upper bound of
$\text{\sc Rouge}_n$.
In the algorithm, $U$ indicates the upper bound score of $\text{\sc Rouge}_n$.
We first set the initial score of upper bound $U$ to
$\text{\sc Rouge}_n(\boldsymbol{R},V)$ (line 3). Then
we compute the density of the $\text{\sc Rouge}'_n$ scores ($\text{\sc
Rouge}_n'(\boldsymbol{R},V,\{w\})/\ell(w)$) for each sentence $w$ in $W$ and
sort them in descending order (lines 4 to 6).
When we have room to add $w$ to the summary, we update $U$ by adding the
$\text{\sc Rouge}'_n(\boldsymbol{R},V,\{w\})$ (line 10) and update length constraint
$L_{\rm max}$
(line 11).
When we do not have room to add $w$, we update $U$ by adding 
the score obtained by multiplying the density of $w$ by the remaining length,
$L_{\rm max}$ (line 13), and exit the while loop.

\subsection{Initial Score for Search}

\begin{algorithm}[tb]
 \caption{Greedy algorithm to obtain initial score}
 \label{findlower}
 \footnotesize
 \begin{algorithmic}[1]
  \FUNC {\text{\sc Greedy}($\boldsymbol{R},D,L_{\rm max}$)}
  \STATE $L \leftarrow 0, S \leftarrow \phi,E \leftarrow D$
  \WHILE{$E \neq \phi$}
  \STATE $s^*{\leftarrow}\displaystyle\mathop{\rm arg~max}_{s \in E}
 {\tiny \left\{\frac{\text{\sc Rouge}_n(\boldsymbol{R},S \cup \{s\}){-}\text{\sc Rouge}_n(\boldsymbol{R},S)}{\ell(\{s\})}\right\}}$
  \STATE $L \leftarrow L+\ell(\{s^*\})$
  \IF {$L \le L_{\rm max}$}
  \STATE $S \leftarrow S \cup \{s^*\}$
  \ENDIF
  \STATE $E \leftarrow E \setminus \{s^*\}$
  \ENDWHILE
  \STATE $i^* \leftarrow \displaystyle\mathop{\rm arg~max}_{i \in D,
  \ell(\{i\}) \le L_{\rm max}} \text{\sc Rouge}_n(
  \boldsymbol{R},\{i\})$
  \STATE $S^* \leftarrow \displaystyle\mathop{\rm arg~max}_{K \in \{\{i^*\},S\}}
  \text{\sc Rouge}_n(\boldsymbol{R},K)$
  \STATE {\bf return} $\text{\sc Rouge}_n(\boldsymbol{R},S^*)$
  \ENDFUNC
 \end{algorithmic}
\end{algorithm}

Since the branch and bound technique prunes the search by comparing the best
solution found so far with the upper bounds, obtaining a good solution in
the early stage is critical for raising search efficiency.

Since $\text{\sc Rouge}_n$ is a monotone submodular function \cite{hlin11}, we can obtain a good approximate solution by a greedy algorithm \cite{khuller99}. 
It is guaranteed that the score of the obtained approximate solution is
larger than $\frac{1}{2}(1-\frac{1}{e})\text{OPT}$, where OPT is the
score of the optimal solution.
We employ the solution as the initial {\sc Rouge}$_n$ score of the candidate oracle summary.

Algorithm \ref{findlower} shows the greedy algorithm. 
In it, $S$ denotes a summary and $D$ denotes a set
of sentences.
The algorithm iteratively adds sentence $s^*$
that yields the largest gain in the $\text{\sc Rouge}_n$ score to current
summary $S$,  provided the length of the summary does not violate 
length constraint $L_{\rm max}$ (line 4). 
After the while loop, the algorithm compares the $\text{\sc Rouge}_n$
score of $S$ with  the maximum $\text{\sc Rouge}_n$ score of the single sentence 
and outputs the larger of the two scores (lines 11 to 13).

\subsection{Enumeration of Oracle summaries}

\begin{algorithm}[tb]
 \caption{Branch and bound technique to enumerate oracle summaries}
 \label{findoracle}
 \footnotesize
 \begin{algorithmic}[1]
  \STATE Read $\boldsymbol{R}$,$D$,$L_{\rm max}$
  \STATE $\tau \leftarrow \text{\sc Greedy}(R,D,L_{\rm max})$,$O_{\tau} \leftarrow \phi$
  \FOR {{\bf each } $s \in D$ }
  \STATE append($S$,$\langle \text{\sc Rouge}_n(\boldsymbol{R},\{s\}),s \rangle$)
  \ENDFOR
  \STATE sort($S$,'descend') 
  \STATE {\bf call} {\sc FindOracle}($S,C$)
  \STATE output $O_{\tau}$
  \PROC {{\sc FindOracle}($Q,V$)}
  \WHILE{$Q \neq \phi$}
  \STATE $s \leftarrow $shift($Q$)
  \STATE append($V,s$)
  \IF {$L_{\rm max}-\ell(V) \ge 0$}
  \IF {$\text{\sc Rouge}_n(\boldsymbol{R},V) \ge \tau$}
  \STATE {$\tau \leftarrow \text{\sc Rouge}_n(\boldsymbol{R},V)$}
  \STATE append($O_{\tau}, V)$
  \STATE {{\bf call} {\sc FindOracle}($Q,V$)}
  \ELSIF {$\widehat{\text{\sc Rouge}_n}(\boldsymbol{R},V) \ge \tau$}
  \STATE {{\bf call} {\sc FindOracle}($Q,V$)}
  \ENDIF
  \ENDIF
  \STATE pop($V$)
  \ENDWHILE
  \ENDPROC
 \end{algorithmic}
\end{algorithm}

By introducing threshold $\tau$ as the best $\text{\sc Rouge}_n$ score
in the search history, pruning decisions involve the following three conditions:
\begin{enumerate}
 \setlength{\itemsep}{-0.1cm}
 \item $\text{\sc Rouge}_n(\boldsymbol{R},V) \ge \tau$;
 \item $\text{\sc Rouge}_n(\boldsymbol{R},V)< \tau$, $\widehat{\text{\sc Rouge}_n}(\boldsymbol{R},V)< \tau$;
 \item $\text{\sc Rouge}_n(\boldsymbol{R},V)< \tau$, $\widehat{\text{\sc Rouge}_n}(\boldsymbol{R},V)\ge \tau$.
\end{enumerate}

With case 1, we update the oracle summary as $V$ and continue the
search. With case 2, because both $\text{\sc Rouge}_n(\boldsymbol{R},V)$ and $\widehat{\text{\sc Rouge}_n}(\boldsymbol{R},V)$ are smaller than $\tau$, the subtree whose root node is the current node (last visited node) is pruned from the search space, and we continue the depth-first search from the neighbor node.
With case 3, we do not update oracle summary as $V$ because
$\text{\sc Rouge}_n(\boldsymbol{R},V)$ is less than $\tau$. However, we might obtain a better oracle summary by continuing the depth-first search because the upper bound of the $\text{\sc Rouge}_n$ score exceeds $\tau$. Thus, we continue to search
for the descendants of the current node.

Algorithm \ref{findoracle} shows the pseudocode that enumerates the oracle summaries.
The algorithm reads a set of reference summaries $\boldsymbol{R}$, length limitation
$L_{\rm max}$, and set of sentences $D$ (line 1) and 
initializes threshold $\tau$ as the $\text{\sc Rouge}_n$ score
obtained by the greedy algorithm (Algorithm \ref{findlower}).
It also initializes $O_{\tau}$, which stores oracle summaries whose $\text{\sc
Rouge}_n$ scores are $\tau$, and priority queue $C$,
which stores the history of the depth-first search (line 2).    
Next, the algorithm computes the $\text{\sc Rouge}_n$ score for each sentence
 and stores $S$ after sorting them in descending order. 
 After that, we start a depth-first search by recursively calling procedure {\sc
FindOracle}. In the procedure, we extract the top sentence
from priority queue $Q$ and append it to priority queue $V$ (lines 11 to 12).
When the length of $V$ is less than $L_{\rm max}$, if $\text{\sc Rouge}_n(\boldsymbol{R},V)$
is larger than threshold $\tau$ (case 1), we update $\tau$ as the score and
append current $V$ to $O_{\tau}$. Then we continue the depth-first search by
calling the procedure the {\sc FindOracle} (lines
15 to 17).
 If $\widehat{\text{\sc Rouge}_n}(\boldsymbol{R},V)$ is larger than $\tau$ (case 3), we do not update $\tau$ and $O_{\tau}$ but reenter the
depth-first search by calling the procedure again (lines 18 to 19).
If neither case 1 nor case 3 is true, we delete the last visited sentence from $V$ and return to the top of the recurrence.

\begin{table}[tb]
 \begin{center}
  {\tabcolsep=0.8mm 
  \begin{tabular}{l|llllll}
   Year & Topics & Docs. & Sents. & Words & Refs. & Length\\
   \hline
   01 & 30 & 10 & 365 & 7706 & 89 & 100\\
   02 & 59 & 10 & 238 & 4822 & 116 & 100\\
   03 & 30 & 10 & 245 & 5711 & 120 & 100\\
   04 & 50 & 10 & 218 & 4870 & 200 & 100\\
   05 & 50 & 29.5  & 885 & 18273.5  & 300 & 250\\
   06 & 50 & 25 &  732.5 & 15997.5 &200 & 250\\
   07 & 45 & 25 & 516 & 11427 & 180 & 250\\
  \end{tabular}
}
  \caption{Statistics of data set}
  \label{stats_multi}
 \end{center}
\end{table}

\section{Experiments}

\subsection{Experimental Setting}

We conducted experiments on the corpora developed for a multiple document
summarization task in DUC 2001 to 2007.
Table \ref{stats_multi} show the statistics of the data.
In particular, the DUC-2005 to -2007 data sets not only have very large
numbers of sentences and words but also a long target length (the
reference summary length) of 250 words.

All the words in the documents were stemmed by Porter's stemmer \cite{porter80}.
We computed $\text{\sc Rouge}_{1}$ scores, excluding stopwords,
and computed $\text{\sc Rouge}_{2}$ scores, keeping
them.
\newcite{owczarzak-EtAl:2012:WEAS} suggested using $\text{\sc Rouge}_{1}$ and keeping stopwords.
However, as Takamura et al. argued \cite{takamura09}, the summaries optimized with non-content words failed to consider the actual quality. Thus, we
excluded stopwords for computing the $\text{\sc Rouge}_{1}$ scores.

We enumerated the following two types of oracle summaries: those for a set of references for a given topic and those for each reference in the set of references. 

\subsection{Results and Discussion}

\subsubsection{Impact of Oracle {\sc Rouge}$_n$ scores}

\begin{table*}[tb]
 \begin{center}
  {\tabcolsep=1.1mm 
  \begin{tabular}{l|rr|rr|rr|rr|rr|rr|rr}
   &
   \multicolumn{2}{c|}{01} &
   \multicolumn{2}{c|}{02} &
   \multicolumn{2}{c|}{03} &
   \multicolumn{2}{c|}{04} &
   \multicolumn{2}{c|}{05} &
   \multicolumn{2}{c|}{06} &
   \multicolumn{2}{c}{07} \\
   & R$_1$ & R$_2$ & R$_1$ & R$_2$ & R$_1$ & R$_2$
				   & R$_1$ & R$_2$ & R$_1$ & R$_2$ & R$_1$ & R$_2$ & R$_1$ &R$_2$\\
   \hline
   Oracle (multi)   & .400 & .164 & .452 & .186 & .434 & .185 & .427 & .162 &
   .445 & .177   & .491 & .211  & .506 & .236\\
   Oracle (single)   & .500 & .226 & .515 & .225 & .525 & .258 & .519 & .228 &
   .574 & .279   & .607 & .303  & .622 & .330\\
   Greedy     & .387 & .161 & .438 & .184 & .424 & .182
               & .412 & .157 & .430 & .173 & .473 & .206 & .495 & .234\\
   Peer         &  .251 & .080 & .269 & .080 & .295 & .094 &
			       .305 & .092 & .262 & .073 & .305 & .095 & .363 & .117\\
   ID      & T & T  & 19 & 19 & 26 & 13 & 67 &  65 & 10 & 15 & 23 & 24
   &29 & 15\\
  \end{tabular}
}
  \caption{{\sc Rouge}$_{1,2}$ scores of oracle summaries, greedy summaries, and system summaries for each data set}
  \label{oracle_scores}
 \end{center}
\end{table*}

Table \ref{oracle_scores} shows the average $\text{\sc Rouge}_{1,2}$ scores of
 the oracle summaries obtained from both a set of references and each reference 
 in the set (``multi'' and ``single''), those of
 the best conventional system (Peer), and those obtained from
 summaries produced by a greedy algorithm (Algorithm \ref{findlower}).

Oracle (single) obtained better $\text{\sc Rouge}_{1,2}$ scores than
Oracle (multi). The results imply that it is easier to optimize a
reference summary than a set of reference summaries. 
On the other hand, the $\text{\sc Rouge}_{1,2}$ scores of these oracle
summaries are significantly higher than those of the best systems.
The best systems obtained $\text{\sc Rouge}_1$ scores
from 60\% to 70\% in ``multi'' and from 50\% to 60\% in ``single''
as well as $\text{\sc Rouge}_2$ scores from 40\% to 55\% in ``multi''
and from 30\% to 40\% in ``single'' for their oracle summaries.

Since the systems in Table \ref{oracle_scores} were developed over
many years, we compared the $\text{\sc Rouge}_n$ scores of the oracle summaries with 
 those of the current state-of-the-art systems using the DUC-2004 corpus and
obtained summaries generated by different systems from a public 
repository\footnote{\url{http://www.cis.upenn.edu/~nlp/corpora/sumrepo.html}}
\cite{HONG14LREC}.
The repository includes 
 summaries produced by the following seven state-of-the-art summarization systems: CLASSY04 \cite{classy04}, CLASSY11
\cite{classy11}, Submodular 
\cite{hui2012-submodular-shells-summarization},
 DPP \cite{Kulesza}, RegSum \cite{hong-nenkova:2014:EACL}, OCCAMS\_V
 \cite{icdmws:DavisCS12,conroy-EtAl:2013:MultiLing}, and ICSISumm
 \cite{gillick-favre:2009:ILPNLP,gillick:tac:09}.
Table \ref{results:duc04}
shows the results.

Based on the results, RegSum \cite{hong-nenkova:2014:EACL} achieved the best $\text{\sc Rouge}_{1}{=}0.331$ result, while ICSISumm
\cite{gillick-favre:2009:ILPNLP,gillick:tac:09} (a compressive
summarizer) achieved the best result with $\text{\sc Rouge}_{2}{=}0.098$.
These systems outperformed the best systems (Peers 65 and 67 in Table \ref{oracle_scores}),
but the differences in the $\text{\sc Rouge}_n$ scores between the systems
and the oracle summaries are still large.
More recently,
\newcite{hong-marcus-nenkova:2015:EMNLP} demonstrated that their system's
combination approach achieved the current best $\text{\sc Rouge}_{2}$
score, 0.105, for the DUC-2004 corpus.
However, a large difference remains between the $\text{\sc Rouge}_2$
score of oracle and their summaries.

In short, the $\text{\sc Rouge}_n$ scores of the oracle summaries are significantly higher than those of the current state-of-the-art summarization systems, both extractive and compressive summarization.
 These results imply that further
 improvement of the performance of extractive summarization is possible.

On the other hand, the $\text{\sc Rouge}_n$ scores of the oracle summaries
are far from $\text{\sc Rouge}_n=1$.
We believe that the results are related to the summary's compression rate.
The data set's compression rate was only 1 to 2\%.
Thus, under tight length constraints, extractive summarization
basically fails to cover large numbers of n-grams in the reference
summary.
This reveals the limitation of the extractive summarization paradigm
 and suggests that we need another direction, compressive or abstractive
 summarization, to overcome the limitation.

\begin{table}[tb]
 \begin{center}
  \begin{tabular}{l|ll}
   System & $\text{\sc Rouge}_{1}$ & $\text{\sc Rouge}_{2}$\\
   \hline
   Oracle (multi)    & .427 & .162\\
   Oracle (single)   & .519 & .228\\
   \hline
   CLASSY04    & .305 & .0897\\
   CLASSY11    & .286 & .0919\\
   Submodular  & .300 &  .0933\\
   DPP         & .309 &  .0960\\
   RegSum      & {\bf .331}  & .0974 \\
   OCCAMS\_V   & .300 & .0974\\
   ICSISumm    & .310 & {\bf .0980}\\

  \end{tabular}
  \caption{$\text{\sc Rouge}_{1,2}$ scores for state-of-the-art
  summarization systems on DUC-2004 corpus}
  \label{results:duc04}
 \end{center}
\end{table}

\subsubsection{{\sc Rouge} Scores of Summaries Obtained from Greedy Algorithm}

Table \ref{oracle_scores} also shows the {\sc Rouge}$_{1,2}$ scores of the summaries
obtained from the greedy algorithm (greedy summaries). Although there are
statistically significant differences between the {\sc Rouge} scores of the oracle
summaries and greedy summaries, those obtained from the 
 greedy summaries achieved near
optimal scores, {\it i.e.}, approximation ratio of them are close to 0.9.
These results are surprising since 
the algorithm's theoretical lower bound is
$\frac{1}{2}(1-\frac{1}{e})(\simeq 0.32)$OPT.

On the other hand, the results do not support that the differences between them are
small at the sentence-level. Table \ref{jaccard} shows the average Jaccard Index between the
oracle summaries and the corresponding greedy summaries
for the DUC-2004 corpus. The results demonstrate that the oracle summaries are
much less similar to the greedy summaries at the sentence-level. Thus, it
might not be appropriate to use greedy summaries as training data for
learning-based extractive summarization systems.

\begin{table}[tb]
 \begin{center}
  \begin{tabular}{l|ll}
    & single & multi\\
   \hline
   {\sc Rouge}$_1$     & .451 & .419\\
   {\sc Rouge}$_2$     & .536 & .530\\
  \end{tabular}
  \caption{Jaccard Index between both oracle and greedy summaries}
  \label{jaccard}
 \end{center}
\end{table}

\subsubsection{Impact of Enumeration}

\begin{table*}[tb]
 \begin{center}
  \begin{tabular}{l|rr|rr|rr|rr}
   &\multicolumn{4}{c|}{Median}  &\multicolumn{4}{c}{Rate}\\
   &\multicolumn{2}{c}{single}  &\multicolumn{2}{c|}{multi}
   &\multicolumn{2}{c}{single}  &\multicolumn{2}{c}{multi}\\
   &   {\sc Rouge}$_1$ & {\sc Rouge}$_2$  &   {\sc Rouge}$_1$ & {\sc Rouge}$_2$
     &   {\sc Rouge}$_1$ & {\sc Rouge}$_2$  &   {\sc Rouge}$_1$ & {\sc Rouge}$_2$\\
   \hline
01   &  8    & 9    & 4 & 5    & .854   &  .787 & .833 & .733 \\
02   &  7.5  & 5.5  & 4 & 4    & .897   &  .836 & .814 & .780 \\
03   &  8    & 10.5 & 3.5 & 4  &  .833   &  .858 & .800 & .900   \\
04   &  8    & 8    & 3.5 & 3  &  .865   &  .865  & .780 & .760 \\
05   &  35   & 35.5 & 2   & 3  &  .916   &  .907  & .580 & .660 \\
06  &  28  & 22     & 2.5 & 3  &    .877    &  .880 & .700 & .720  \\
07   &  23   & 16   & 4   & 2  &  .910    &  .878  & .733 & 711  \\
  \end{tabular}
  \caption{Median number of oracle summaries and rates of reference
  summaries and topics with multiple oracle summaries for each data set}
  \label{agreement}
 \end{center}
\end{table*}

Table \ref{agreement} shows the median number of oracle summaries and
the rates of the reference summaries that have multiple
oracle summaries for each data set.
Over 80\% of the reference summaries and about 
60\% to 90\% of the topics have
multiple oracle summaries.
Since the $\text{\sc Rouge}_n$ scores are based on the unweighted counting of n-grams,
when many sentences have similar meanings, {\it i.e.,}
many redundant sentences, the
number of oracle summaries that have the same
$\text{\sc Rouge}_n$ scores increases.
The source documents of multiple document summarization tasks
are prone to have many such redundant sentences, and
the amount of oracle summaries is large.

The oracle summaries offer significant
benefit with respect to evaluating the extracted sentences.
Since both the oracle and system summaries are sets of sentences,
it is easy to check whether each sentence in the system summary is contained in one of the oracle summaries. 
Thus, we can exploit the F-measures, which are useful for evaluating
classification-based extractive summarization \cite{Mani:1998,osborne:2002,hirao02}.
Here, we have to consider that the oracle summaries, obtained from a
reference summary or a set of reference summaries, are not identical at the
sentence-level (e.g., the average Jaccard Index between the oracle summaries
for the DUC-2004 corpus is around 0.5).
The F-measures are varied with the oracle summaries that are used for such computation.
For example, assume that we have system summary $S{=}\{s_1,s_2,s_3,s_4\}$ and
oracle summaries $O_1{=}\{s_1,s_2,s_5,s_6\}$ and 
$O_2{=}\{s_1,s_2,s_3\}$. 
The precision for $O_1$  is 0.5, while that for $O_2$ is 0.75; the recall
for $O_1$ is 0.5, while that for $O_2$ is 1; the F-measure for $O_1$ is 0.5,
while that for $O_2$ is 0.86.

Thus, we employ the scores gained by averaging all of the oracle summaries as evaluation measures.
Precision, recall, and F-measure are defined as follows:
$P{=}\{\sum_{O \in O_{\rm all}}|O \cap S|/|S|\}/|O_{\rm all}|$,
$R{=}\{\sum_{O \in O_{\rm all}}|O \cap S|/|O|\}/|O_{\rm all}|$,
$\text{F-measure}{=}2PR/(P+R)$.

To demonstrate F-measure's effectiveness, we investigated the
correlation between an F-measure and human judgment based on the evaluation
results obtained from the DUC-2004 corpus.
The results include summaries generated by 17 systems, each of which has a mean coverage score assigned by a human subject.
We computed the correlation coefficients between the average F-measure
  and the average mean coverage score for 50 topics.
Table \ref{correl}
shows Pearson's $r$ and Spearman's $\rho$. In the table, ``F-measure (R$_1$)'' and
``F-measure (R$_2$)'' indicate the F-measures calculated using oracle summaries optimized to
$\text{\sc Rouge}_1$ and $\text{\sc Rouge}_2$, respectively. ``M''
indicates the F-measure calculated using multiple oracle summaries, and ``S''
indicates F-measures calculated using randomly selected oracle
summaries. ``multi'' indicates oracle summaries obtained from a set of
references, and ``single'' indicates oracle summaries obtained from a
reference summary in the set. 
For ``S,''  we randomly selected a single oracle summary and calculated
the F-measure 100 times and took the average value with
 the 95\% confidence interval of the F-measures by bootstrap resampling.

The results demonstrate that the F-measures are strongly correlated with human judgment. Their values are comparable with those of $\text{\sc Rouge}_{1,2}$. In particular,
F-measure (R$_1$) (single-M) achieved the best Spearman's $\rho$ result.
When comparing ``single'' with ``multi,'' Pearson's $r$ of
``multi'' was slightly lower than that of ``single,'' and the Spearman's $r$
of ``multi'' was almost the same as those of ``single.''
``M'' has significantly
better performance than ``S.''
 These results imply that F-measures based on oracle summaries are a
 good evaluation measure and that oracle summaries have the potential to
 be an alternative to human-made reference summaries in terms of
 automatic evaluation.
Moreover, the enumeration of the oracle summaries for a given reference summary
or a set of reference summaries is essential for automatic evaluation. 

\subsubsection{Search Efficiency }

\begin{table}[tb]
 \small
 \begin{center}
  \begin{tabular}{l|ll}
   Metric & $r$ & $\rho$\\
   \hline
   $\text{\sc Rouge}_1$ & .861 & .760\\
   $\text{\sc Rouge}_2$ & {\bf .907} & .831\\
   \hline
F-measure (R$_1$) (single-M)  & .857 & {\bf .855}\\
F-measure (R$_1$) (single-S)  & .815-.830 & .811-.830\\
F-measure (R$_2$) (single-M)  & .904 & .826\\
F-measure (R$_2$) (single-S)  & .855-.865 & .740-.760\\
F-measure (R$_1$) (multi-M)  & .814 & .841\\
F-measure (R$_1$) (multi-S)  & .794-.802 & .803-.813\\
F-measure (R$_2$) (multi-M)  & .824 & .846\\
F-measure (R$_2$) (multi-S)  & .806-.816 & .797-.817\\
  \end{tabular}
  \caption{Correlation coefficients between automatic evaluations and
  human judgments on DUC-2004 corpus}
  \label{correl}
 \end{center}
\end{table}

To demonstrate the efficiency of our search algorithm against
the naive exhaustive search method, we compared the number of 
feasible solutions (sets of sentences that satisfy the length
constraint) with the number of summaries that were checked in our search
algorithm. 
The algorithm that counts the number of
feasible solutions is shown in Appendix B.

Table \ref{order} shows the median number of feasible solutions and
checked summaries yielded by our method for each data set (in the case
of ``single'').
The differences in the number of feasible solutions between $\text{\sc Rouge}_1$
and $\text{\sc Rouge}_2$ are very large. 
Input set ($|D|$) of $\text{\sc Rouge}_1$ is much larger than $\text{\sc
Rouge}_1$.
On the other hand, the differences between $\text{\sc Rouge}_1$ and
$\text{\sc Rouge}_2$ in our method
are of the order of $10$ to $10^2$. When comparing our method with naive exhaustive searches, its search space is significantly smaller.
The differences are of the order of $10^7$ to $10^{30}$ with $\text{\sc Rouge}_1$ and $10^4$
to $10^{17}$ with $\text{\sc Rouge}_2$.
These results demonstrate the efficiency of our branch and bound technique.

In addition, we show an example of the processing time for extracting one oracle summary and enumerating all of the oracle
summaries for the reference summaries in the DUC-2004 corpus with a Linux
machine (CPU:~Intel$^{\textregistered}$ Xeon$^{\textregistered}$ X5675
(3.07GHz)) with 192 GB of RAM. We utilized {\tt CPLEX 12.1} to solve the ILP
problem.  Our algorithm was implemented in C++ and
complied with GCC version 4.4.7.
The results show that we needed 0.026 and 0.021 sec. to extract one oracle
summary per reference summary and 0.047 and 0.031 sec. to extract
one oracle summary per set of reference summaries 
for $\text{\sc Rouge}_1$ and $\text{\sc Rouge}_2$, respectively.
We needed 11.90 and 1.40
sec. to enumerate the oracle summaries per reference summary and 102.94
and 3.65 sec. per set of reference summaries for $\text{\sc
Rouge}_1$ and $\text{\sc Rouge}_2$, respectively.
The extraction of one oracle summary for a reference summary
can be achieved with the ILP solver in practical time and 
the enumeration of oracle summaries is also efficient.
However, to enumerate oracle summaries, we needed several weeks
for some topics in DUCs 2005 to 2007 since they hold a huge number of source sentences. 

\section{Conclusions}

\begin{table}[tb]
 \begin{center}
  {\tabcolsep=0.9mm 
  \begin{tabular}{l|ll|ll}
     &  \multicolumn{2}{c|}{{\sc Rouge}$_1$} & \multicolumn{2}{c}{{\sc Rouge}$_2$}\\
   & Naive & Proposed 
   & Naive & Proposed\\
   \hline
01   & 3.66$\times 10^{13}$
       & 5.75$\times 10^3$
	   & 3.32$\times 10^7$
	       & 1.00$\times 10^3$\\
02   & 1.12$\times 10^{12}$
       & 4.64$\times 10^3$
	   & 1.34$\times 10^7$
	       & 8.87$\times 10^2$\\
03 & 1.62$\times 10^{11}$
       & 3.65$\times 10^3$
	   & 6.37$\times 10^6$
	       & 8.19$\times 10^2$\\
04   & 9.65$\times 10^{10}$
       & 4.47$\times 10^3$
	   & 6.90$\times 10^6$
	       & 9.83$\times 10^2$\\
05   & 5.48$\times 10^{36}$
       & 2.32$\times 10^6$
	   & 3.48$\times 10^{21}$
	       & 7.03$\times 10^4$\\
06  & 1.94$\times 10^{32}$
       & 1.97$\times 10^6$
	   & 2.11$\times 10^{20}$
	       & 5.08$\times 10^4$\\
07   & 4.14$\times 10^{28}$
       & 1.40$\times 10^6$
	   & 1.81$\times 10^{19}$
	       & 2.60$\times 10^4$\\
  \end{tabular}
}
  \caption{Median number of summaries checked by each search method}
  \label{order}
 \end{center}
\end{table}

To analyze the limitations and the future direction of extractive
summarization, this paper proposed (1) Integer Linear Programming (ILP) formulation
to obtain extractive oracle summaries in terms of {\sc Rouge}$_n$ scores
 and  (2) an algorithm that enumerates all oracle summaries to exploit
 F-measures that evaluate 
 the sentences extracted by systems. 

The evaluation results obtained from the corpora of DUCs 2001 to 2007 
 identified the following:
 (1)
room still exists to improve the $\text{\sc Rouge}_n$ scores of
 extractive summarization systems even though the
$\text{\sc Rouge}_n$ scores of the oracle summaries fell below the
theoretical upper bound $\text{\sc Rouge}_n{=}1$.
(2) Over 80\% of the reference summaries and from 60\% to 90\% of the sets of
 reference summaries 
have multiple oracle summaries, and the F-measures computed by
utilizing the enumerated oracle
summaries showed  stronger correlation with human judgment than those
computed from single oracle summaries.

\section*{Appendix A.}

\begin{proof}
We can rewrite the right side of equation (\ref{our_rouge}) as follows:
\begin{equation}
\begin{split}
\text{\sc Rouge}(\boldsymbol{R},V){+}\text{\sc Rouge}'_n(\boldsymbol{R},V,W)=&\\
& \kern-15em
 \frac{
 \displaystyle
 \sum_{k=1}^{|\boldsymbol{R}|}\sum_{t_n \in U({\cal R}_k)} f(t_n,{\cal
 R}_k,{\cal V,W})}{
 \displaystyle
\sum_{k=1}^{|\boldsymbol{R}|} \sum_{t_n
 \in U({\cal R}_k)} N(t_n,{\cal R}_k)}.
\end{split}
 \end{equation}
Here,  $f(t_n,{\cal R}_k,{\cal V,W})$ is defined as follows:
\begin{equation}
\label{function_f}
\begin{split}
f(t_n,{\cal R}_k,{\cal V,W}){=}\min\{N(t_n,{\cal R}_k),N(t_n,{\cal V})\}+&\\
& \kern-15em \min\{N(t_n,{\cal R}_k\setminus {\cal V}),N(t_n,{\cal W})\}.
\end{split}
\end{equation}
$N(t_n,{\cal R}_k\setminus {\cal V})$ is the number of times
$t_n$ occurs in the multiple set ${\cal R}_k \setminus {\cal V}$. Equation
(\ref{function_f}) is rewritten as
\begin{equation}
\label{final}
\begin{split}
f(t_n,{\cal R}_k,{\cal V,W}){=}\min\{N(t_n,{\cal R}_k),N(t_n,{\cal V})\}+&\\
& \kern-20em \min\{\max\{N(t_n,{\cal R}_k){-}N(t_n,{\cal V}),0\},N(t_n,{\cal W})\}.
\end{split}
\end{equation}

The solutions of equation (\ref{final}) are obtained by considering
the following three conditions:

\begin{enumerate}
 \item If $N(t_n,{\cal R}_k)-N(t_n,{\cal V})>0$ and $N(t_n,{\cal
       R}_k)-N(t_n,{\cal V})>N(t_n,{\cal W})$, then $f(t_n,{\cal
       R}_k,{\cal V,W})=N(t_n,{\cal V})+N(t_n,{\cal W})$
 \item If $N(t_n,{\cal R}_k)-N(t_n,{\cal V})>0$ and $N(t_n,{\cal
       R}_k)-N(t_n,{\cal V})<N(t_n,{\cal W})$, then $f(t_n,{\cal
       R}_k,{\cal V,W})=N(t_n,{\cal R}_k)$
 \item If $N(t_n,{\cal R}_k)-N(t_n,{\cal V}) < 0$,
 then $f(t_n,{\cal R}_k,{\cal V,W})=N(t_n,{\cal R}_k)$
\end{enumerate}

From the above relations,

\begin{eqnarray}
f(t_n,{\cal R}_k,{\cal V,W})=\nonumber&\\
& \kern-8em \min\{N(t_n,{\cal R}_k),N(t_n,{\cal V})+N(t_n,{\cal W})\}.
\end{eqnarray}

Thus,

 \begin{equation}
\begin{split}
\text{\sc Rouge}_n(\boldsymbol{R},V{\cup}W)=&\\
 &\kern-10em \frac{
 \displaystyle
 \sum_{k=1}^{|\boldsymbol{R}|}\sum_{t_n{\in}U({\cal R}_k)}\min\{N(t_n,{\cal R}_k),N(t_n,{\cal
 V}){+}N(t_n,{\cal W})\}}{
 \displaystyle
 \sum_{k=1}^{|\boldsymbol{R}|}\sum_{t_n{\in}U({\cal R}_k)} N(t_n,{\cal R}_k)}
\end{split}
 \end{equation}

\end{proof}

\section*{Appendix B.}

\begin{algorithm}[tb]
 \caption{Dynamic Programming Algorithm to Count the Number of the
 Feasible Summaries}
 \label{subsetsum}
\footnotesize
 \begin{algorithmic}[1]
  \FUNC {\text{\sc GetNumFS}($D,L_{\rm max}$)}
  \STATE $C[0][0]\leftarrow 1, C[0][j] \leftarrow 0, 1 \le j \le L_{\rm max}$
  \FOR  {$i=1$ to $|D|$}
 \FOR  {$j=0$ to $L_{\rm max}$}
  \IF {$j-\ell(s_i) \ge 0$}
  \STATE $C[i][j] \leftarrow C[i-1][j]+C[i-1][j-\ell(s_i)]$
  \ELSE
  \STATE $C[i][j] \leftarrow C[i-1][j]$
  \ENDIF
  \ENDFOR
  \ENDFOR
  \STATE {\bf return} $\displaystyle\mathop\sum_{j=1}^{L_{\rm max}} C[|D|][j]$
  \ENDFUNC
 \end{algorithmic}
\end{algorithm}

We propose an algorithm to compute the number of feasible solutions
under the length constraint by extending the dynamic programming based
approach for the subset sum problem \cite{algo}.
We define $C[i][j](0 \le i \le |D|, 0 \le j \le L_{\rm max})$, which stores the
number of feasible solutions (length is less than $j$) that can
be obtained from set $\{s_1,\ldots,s_i\}$ as follows:
\begin{itemize}
\item Initialization
 \begin{equation}
  \begin{array}{lr}
   C[0][j]=0 &  (0 \le j \le L_{\rm max})
  \end{array}
 \end{equation}
        \item Recurrence ($1 \le i \le |D|$)
        \begin{equation}
       \begin{split}
        \kern-1em C[i][j]{=} &\\
        &
        \kern-5em\left\{
        \begin{array}{lr}
         C[i{-}1][j]+C[i{-}1][j{-}\ell(s_i)] &\text{if } j{-}\ell(s_i) \ge 0\\
         C[i{-}1][j] &\text{otherwise}
        \end{array}
        \right.
       \end{split}
        \end{equation}
 \end{itemize}

Algorithm \ref{subsetsum} is a dynamic program that fills out
the ($|D|+1) \times (L_{\rm max}+1$) table. After the table is filled,
each cell on the $|D|+1$-th line stores the number of feasible solutions.
In the algorithm, first, we pick up the sentences that contain
an n-gram that appears in the reference summary at least once and
recursively count the number of feasible solutions.
Then, the sum of the $j$-th line whose index is from 1 to $L_{\rm max}$
indicates the number of feasible solutions.
The order of the algorithm is $O(n L_{\rm max})$.

\bibliographystyle{eacl2017}
\bibliography{mybibfile}

\end{document}